\documentclass[runningheads]{llncs}

\usepackage{eccv}

\usepackage{eccvabbrv}

\usepackage{graphicx}
\usepackage{booktabs}
\usepackage{hyperref}
\usepackage{tikz}
\usepackage{float}
\usepackage{makecell}
\usepackage{overpic}
\usepackage{float}
\usepackage{pict2e}
\usepackage{xspace}
\usepackage[table]{xcolor}
\usetikzlibrary{arrows.meta, positioning, calc}

\usepackage[accsupp]{axessibility}  %

\usepackage{hyperref}

\usepackage{orcidlink}

\usepackage{listings}
\usepackage{lstautogobble}  
\usepackage{color}          
\usepackage{zi4}            
\definecolor{bluekeywords}{rgb}{0.13, 0.13, 1}
\definecolor{greencomments}{rgb}{0, 0.5, 0}
\definecolor{redstrings}{rgb}{0.9, 0, 0}
\definecolor{graynumbers}{rgb}{0.5, 0.5, 0.5}
\begin{document}

\title{PixSDS: Why Latent SDS Makes Noisy Pixels}

\author{Vsevolod Skorokhodov\orcidlink{0009-0002-9840-761X}}

\authorrunning{V. Skorokhodov}

\institute{EPFL, Lausanne, Switzerland\\
\email{vsevolod.skorokhodov@epfl.ch}}

\maketitle

\newcommand{\encdec}[0]{\operatorname{encdec}}
\newcommand{\dec}[0]{\operatorname{dec}}
\newcommand{\enc}[0]{\operatorname{enc}}
\newcommand{\cnorm}[0]{\operatorname{cnorm}}

\newcommand{\ours}[0]{PixSDS\xspace}

\definecolor{best}{RGB}{200, 200, 255}
\definecolor{secondbest}{RGB}{240, 240, 255}

\begin{abstract}

Score Distillation Sampling (SDS) enables text-to-3D generation by optimizing rendered images with a pretrained diffusion prior, but latent SDS often produces structured color artifacts and high-frequency texture noise. We identify a failure mode of latent SDS caused by VAE-induced pixel drift: the optimized image can move along pixel-space directions that are weakly constrained by the VAE encoder, so its latent representation remains clean and semantically meaningful while the image itself accumulates visible artifacts. We support this diagnosis with controlled 2D SDS experiments, VAE-only optimization, and a simplified analysis showing that encoder-like latent objectives can amplify image-space noise when the inverse mapping to pixels is underconstrained. Motivated by this observation, we propose PixSDS, a lightweight VAE-consistent gradient repair method. PixSDS decodes a latent SDS look-ahead step and uses the decoded image as a clean direction for pixel-space optimization, reducing motion in VAE-inconsistent directions without retraining the diffusion model, changing the renderer, or replacing the SDS objective. Experiments in 2D optimization and text-to-3D generation show that PixSDS substantially reduces structured artifacts while preserving semantic content. Code is publicly available at \url{https://sevashasla.github.io/pixsds-webpage/}.

\end{abstract}

\section{Introduction}\label{sec:intro}

Score Distillation Sampling (SDS)~\cite{poole2022dreamfusion} has become a central tool for optimization-based text-to-3D generation.
By using a pretrained text-to-image diffusion model as a prior, SDS makes it possible to optimize 3D representations from text prompts without training a dedicated 3D generative model.
This flexibility has made SDS useful across neural fields~\cite{poole2022dreamfusion, wang2023prolificdreamer}, meshes~\cite{chen2023fantasia3d}, and Gaussian splats~\cite{tang2024dreamgaussian, liang2024luciddreamer, lukoianov2024score}, especially when large-scale paired text--3D data is unavailable.

Despite its success, SDS often produces visually distracting artifacts~\cite{tang2024dreamgaussian, liang2024luciddreamer}, including structured color patterns, high-frequency texture noise, and noisy floating geometry.
These artifacts are especially common in latent-diffusion-based pipelines, where the diffusion prior operates on the latent representation of a variational autoencoder (VAE) rather than directly on pixels.
Prior work has proposed practical remedies, such as modifying the SDS objective~\cite{katzir2023nfsd}, changing the optimization procedure~\cite{tang2024dreamgaussian}, or clipping large pixel-space gradients~\cite{pan2024pgc}.
However, these methods mostly address the symptoms of the artifact, while the underlying cause remains unclear.

In this paper, we argue that a key source of these structured artifacts is the mismatch between pixel-space optimization and the VAE representation used by latent diffusion models.
Latent SDS constrains an optimized rendering through its encoded latent, but the inverse mapping from latents to pixels is underconstrained.
As a result, the optimized image can drift along pixel-space directions that are weakly visible to the VAE encoder: the latent code remains clean and semantically valid, while the directly optimized pixel image accumulates structured noise.

We support this view with controlled 2D experiments, VAE-only optimization, and a simplified theoretical model.
Together, these results show that structured artifacts can arise even without 3D rendering or texture extraction, and that encoder-like latent objectives can amplify image-space noise when the inverse mapping to pixels is underconstrained.
This diagnosis suggests that artifact reduction should not merely suppress large gradients, but should instead guide pixel-space updates toward directions that are consistent with the VAE latent space.

Motivated by this observation, we introduce \ours, a lightweight VAE-consistent gradient repair method for latent SDS.
At each optimization step, \ours decodes the latent-space SDS update and uses the decoded image as a clean direction for pixel-space optimization.
This repair does not require retraining the diffusion model, modifying the renderer, or changing the underlying 3D representation, and can be combined with existing SDS-style objectives.
Experiments in controlled 2D optimization and text-to-3D generation show that \ours substantially reduces structured artifacts while preserving semantic content.

In summary, our contributions are:
\begin{enumerate}
\item{
We identify VAE-induced pixel drift as a failure mode of latent score distillation, where pixel-space parameters accumulate structured high-frequency artifacts while their VAE latents remain clean and semantically valid.
}
\item{
We provide controlled evidence and a simplified theoretical analysis showing that encoder-like latent objectives can amplify image-space noise when the inverse mapping from latents to pixels is underconstrained.
}
\item{
We introduce \ours, a lightweight VAE-consistent gradient repair method that decodes the latent-space SDS step and uses it as a clean direction for pixel-space optimization.
}
\item{
We show that \ours reduces structured artifacts in controlled 2D SDS optimization and improves visual cleanliness when integrated into existing text-to-3D pipelines.
}
\end{enumerate}

\section{Related Works}\label{sec:related_works}

\subsection{Score Distillation Sampling}\label{sec:sds}

DreamFusion~\cite{poole2022dreamfusion} introduced Score Distillation Sampling (SDS) as a way to optimize a 3D representation using a pretrained 2D diffusion prior. Let $x=R_\theta(c)$ denote a rendering of the current 3D representation under camera $c$, and let
\[
    z_t = \alpha_t x + \sigma_t \epsilon
\]
be the corresponding noised sample at timestep $t$. The SDS update can be written as
\begin{equation}\label{eq:sds}
\nabla_\theta \mathcal{L}_{\mathrm{SDS}}
=
\mathbb{E}_{t,\epsilon}
\left[
w(t)
\left(\frac{\partial x}{\partial \theta}\right)^{\!\top}
\bigl(\hat{\epsilon}_\phi(z_t;y,t)-\epsilon\bigr)
\right],
\end{equation}
where $\hat{\epsilon}_\phi$ is the noise predicted by the diffusion model, $y$ is the text prompt, and $w(t)$ is a timestep-dependent weighting term.

Several works improve SDS by changing the distillation objective or modifying its gradient. ProlificDreamer~\cite{wang2023prolificdreamer} introduces Variational Score Distillation, HiFA~\cite{zhu2024hifa} combines latent-space and pixel-space guidance, NFSD~\cite{katzir2023nfsd} removes an undesired noise component from the SDS signal, and PGC~\cite{pan2024pgc} clips decoded pixel-wise gradients. Recent reformulations such as SDI~\cite{lukoianov2024score} and SDS-Bridge~\cite{mcallister2024sdsbridge} derive alternative distillation directions from sampling or transport perspectives. These methods can be written abstractly as
\begin{equation}\label{eq:sds_grad}
    g_\theta
    =
    \mathbb{E}_{t,\epsilon}
    \left[
    w(t)
    \left(\frac{\partial x}{\partial \theta}\right)^{\!\top}
    \Delta_t
    \right],
\end{equation}
where $\Delta_t$ is a method-specific distillation direction; for vanilla SDS, $\Delta_t=\hat{\epsilon}_\phi(z_t;y,t)-\epsilon$.

Our work is complementary to these methods. Rather than deriving a new score distillation objective, we study why latent SDS can produce structured pixel-space artifacts even when the corresponding latent representation remains clean, and use this analysis to repair the pixel-space update.

\subsection{Artifacts and Noise in SDS}\label{sec:sds_noise}

SDS-based optimization often suffers from oversaturation, blurriness, and high-frequency texture artifacts~\cite{katzir2023nfsd, lukoianov2024score}. In this work, we focus on structured color artifacts and spatially coherent noise patterns, as illustrated in~\cref{fig:clean_latents}.

Prior work has attributed such artifacts to different mechanisms. DreamGaussian~\cite{tang2024dreamgaussian} observes that SDS can introduce high-frequency texture noise in 3D Gaussian optimization and relates part of the problem to texture extraction and mipmap sampling. PGC~\cite{pan2024pgc} instead focuses on latent diffusion guidance and argues that decoded SDS gradients can contain large pixel-wise outliers, which motivates clipping the pixel-space gradient.

Our analysis is complementary to these explanations. We show that similar structured artifacts can arise even in purely 2D optimization, without mesh extraction, mipmap sampling, or a 3D renderer. This suggests that renderer-specific effects may amplify the artifacts, but are not necessary for them to occur. We instead link the artifact formation to the underconstrained geometry of the VAE mapping: an optimized image can drift into noisy pixel-space directions while its encoded latent, and the image decoded from that latent, remain clean.

\section{Diagnosing SDS Artifacts}\label{sec:diagnosis}

In this section, we study why latent SDS produces structured pixel-space artifacts. Latent SDS differs from pixel-space SDS in several ways: the optimized signal has a different tensor shape, the diffusion prior operates in a different representation space~\cite{image_id,vae_id}, and the update is propagated through the VAE mapping. We isolate these factors with controlled experiments and show that the VAE mapping is a sufficient mechanism for the observed artifacts.

\begin{figure}[t!]
    \begin{minipage}{0.49\textwidth}
        \centering
        \includegraphics[width=\textwidth]{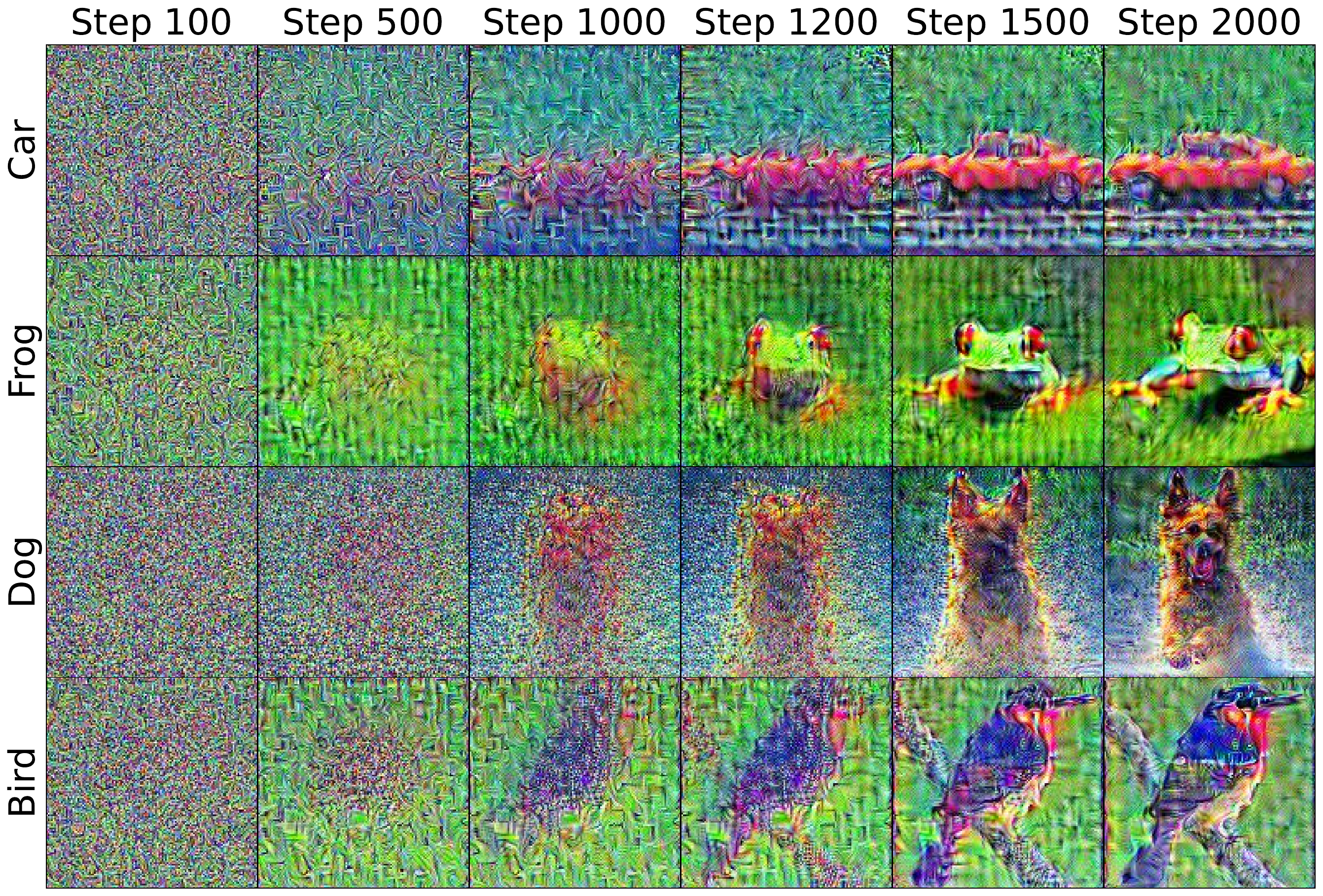}
        \caption{\textbf{Small latent diffusion model.} SDS with a low-resolution latent diffusion model still produces structured artifacts, suggesting that tensor shape alone does not explain the failure mode.}
        \label{fig:sds_nano}
    \end{minipage}
    \hfill
    \begin{minipage}{0.49\textwidth}
        \centering
        \includegraphics[width=\textwidth]{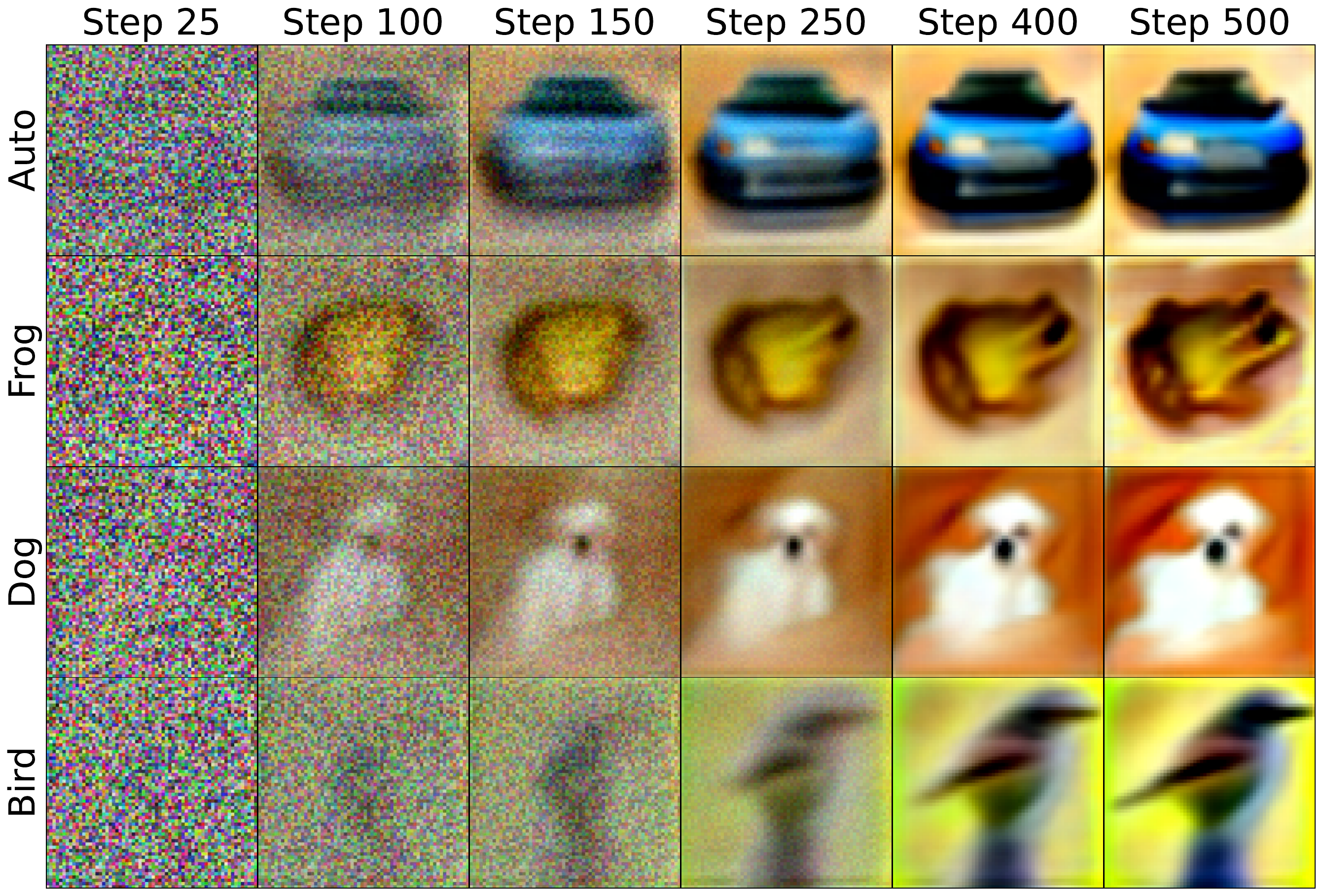}
        \caption{\textbf{Small pixel diffusion model.} SDS with a pixel-space diffusion model produces realistic images without the same structured artifacts, suggesting that pixel-space optimization alone is not the cause.}
        \label{fig:small_dm}
    \end{minipage}
\end{figure}

\begin{figure}[t!]
    \centering
    \includegraphics[width=\textwidth]{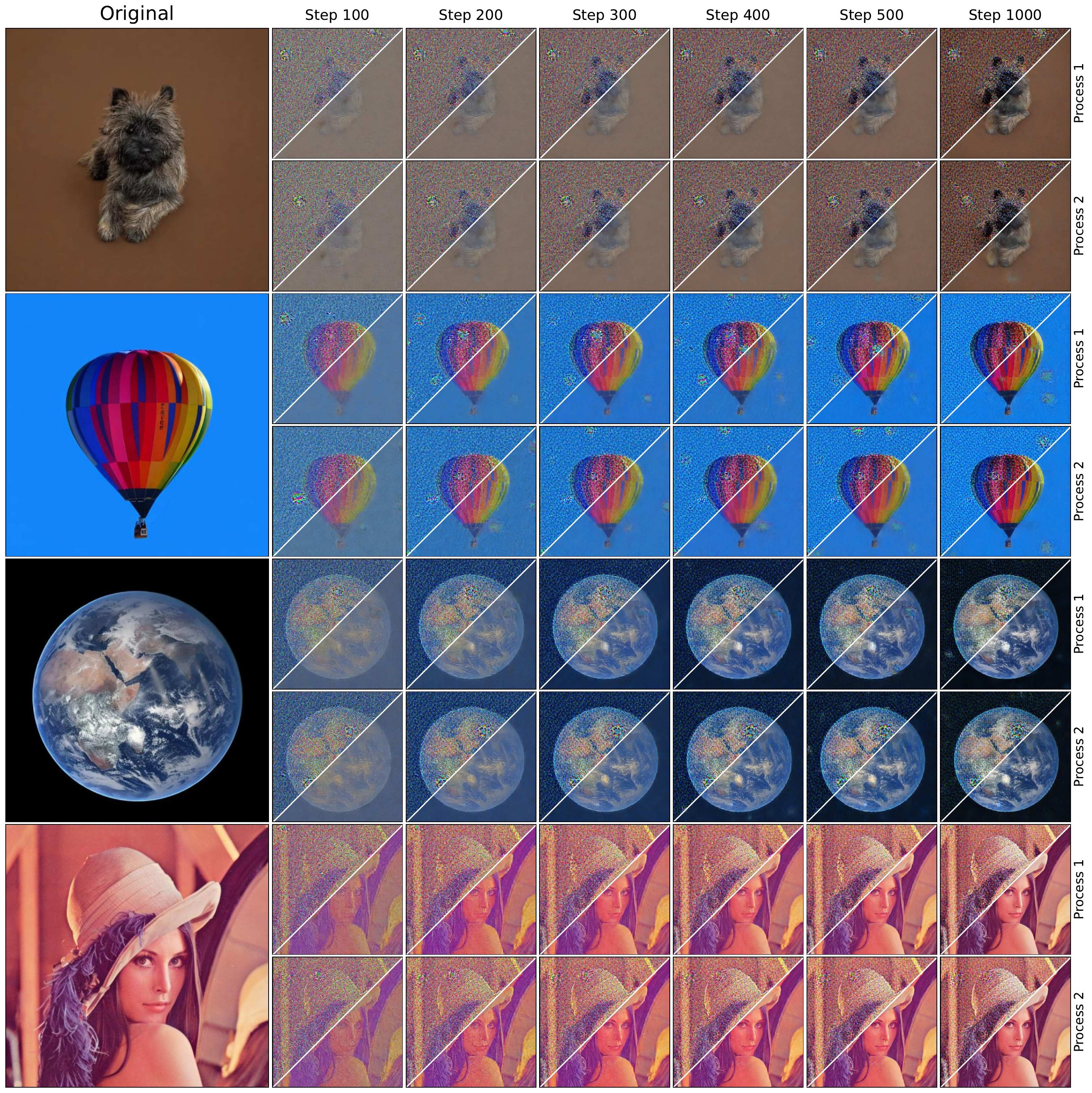}
    \caption{
        \textbf{VAE-only optimization.} Column 1 shows the target image $X$. Columns 2--7 show optimized images $Z_i$ from two different random initializations, with their decoded latents $\dec(\enc(Z_i))$ shown in the bottom-right inset. Even without a diffusion model, optimizing only through the VAE encoder produces structured noise in $Z_i$, while the decoded latents remain clean.
    }
    \label{fig:vae_only_optimization}
\end{figure}

\paragraph{Tensor shape.}
One possible explanation is that artifacts arise because latent diffusion operates on a compressed low-resolution tensor rather than on an image. To test this, we run SDS with \textit{stable-diffusion-nano-2-1}~\cite{guisard_stable_diffusion_nano_2_1}, a Stable Diffusion 2.1 model fine-tuned at $128{\times}128$ resolution. As shown in~\cref{fig:sds_nano}, structured artifacts still appear. This suggests that latent tensor resolution or tensor shape alone does not explain the artifact formation.

\paragraph{Pixel-space optimization.}
Another possibility is that SDS optimization itself is unstable when applied directly to pixels. To test this, we perform SDS in pixel space using a small conditional diffusion model trained on CIFAR-10~\cite{cifar10} at $64{\times}64$ resolution. As shown in~\cref{fig:small_dm}, the optimized images remain realistic and do not exhibit the same structured artifacts as latent SDS. This indicates that optimizing pixels with SDS is not by itself sufficient to produce the observed failure mode.

\paragraph{VAE mapping.}
The remaining difference is the VAE mapping used by latent diffusion models. In latent SDS, the diffusion loss is defined in the latent space, but the optimized variable may live in pixel space and receive updates through the VAE encoder or decoder. This creates an underconstrained inverse problem: many visually different images can map to similar latent codes. As a result, an image can move in directions that are weakly constrained by the latent objective while still preserving a clean latent representation.

To isolate this effect from the diffusion model, we optimize an image using only the VAE encoder. Let
\[
    \enc: \mathbb{R}^{3 \times H \times W} \rightarrow \mathbb{R}^{c \times h \times w}
\]
denote the VAE encoder, and let $X$ be a target image with latent code $x=\enc(X)$. Starting from a randomly initialized image $Z \sim \mathcal{U}(0,1)$, we optimize $Z$ to match the target latent by minimizing
\begin{equation}\label{eq:vae_only_objective}
    \mathcal{L}_{\mathrm{VAE}}(Z)
    =
    \left\| \enc(Z) - x \right\|_2^2
\end{equation}
with SGD and learning rate $0.1$. As shown in~\cref{fig:vae_only_optimization}, the optimized images develop structured noise patterns similar to those observed in latent SDS, even though no diffusion model is used. At the same time, the decoded latents $\dec(\enc(Z))$ remain visually clean. This experiment shows that the VAE mapping alone can produce the clean-latent/noisy-image mismatch.

\paragraph{Toy analogy.}
A simple low-dimensional example illustrates the same issue. Consider minimizing
\[
    \ell(x,y)=(x+2y)^2.
\]
The set of minimizers is the line $x+2y=0$, so the objective does not select a unique solution. Starting from $(x_0,y_0)=(1,1)$, gradient descent converges to $(x^\ast,y^\ast)=(0.4,-0.2)$, although $(0,0)$ is also a minimizer. The solution $(0.4, -0.2)$ is noisier than $(0,0)$. Therefore, the optimization does not necessary converge to a less noisy solution without any additional constraints.

We formalize this intuition in the Appendix (Sec. A). For a 1D convolution with weights $\omega \in \mathbb{R}^{m}$, $m \geq 3$, and $\sum_{i=1}^{m}\omega_i \neq 0$, we show that for any $C>0$ there exists an input $x \in \mathbb{R}^{n}$ whose initial noise is below $C$, but minimizing $\frac{1}{2}\|\omega * x\|_2^2$ by gradient descent increases the noise in the solution. Thus, noise amplification can occur even in simple encoder-like objectives.

Together, these experiments show that structured artifacts are not fully explained by latent tensor shape, pixel-space SDS optimization, or renderer-specific effects. Instead, they identify the VAE mapping as a sufficient and previously underemphasized mechanism: latent objectives can keep the encoded representation clean while allowing the optimized image to drift into noisy pixel-space directions.

\section{Method}\label{sec:method}

In this section, we introduce \ours. The method is motivated by a simple observation: during latent SDS optimization, the optimized image can accumulate structured pixel-space artifacts, while its VAE latent representation and decoded latent remain visually clean. We use this decoded latent update as a VAE-consistent direction to repair the SDS update in pixel space.

\subsection{Clean Latents}\label{sec:clean_latents}

\begin{figure}[t!]
    \vspace{10pt}
    \centering
    \begin{overpic}[width=\textwidth]{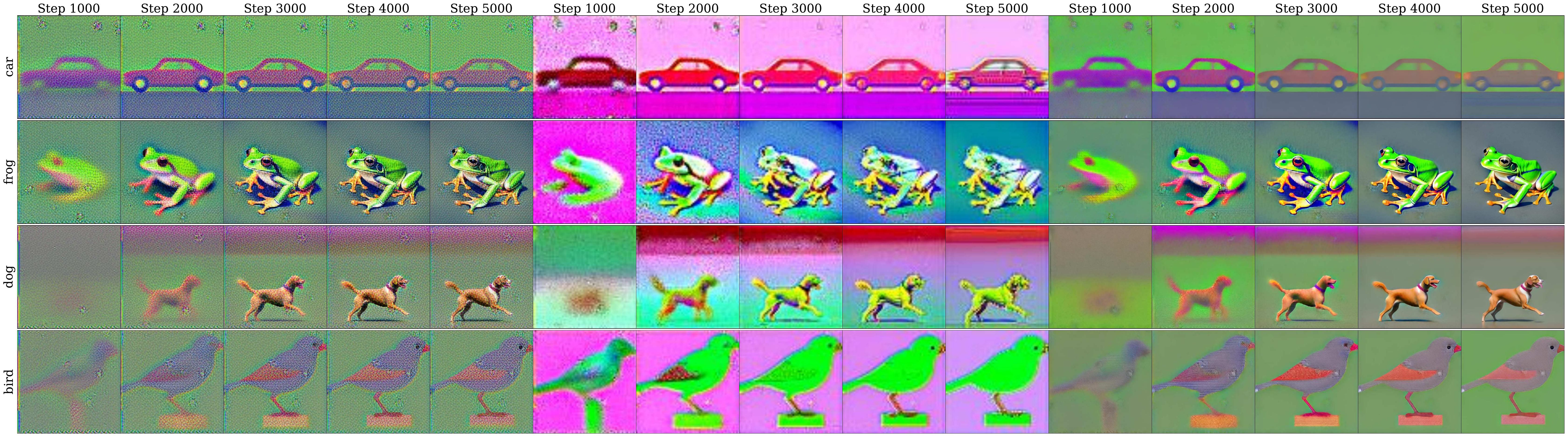}

    \put(2, 30){\vector(1,0){33}}
    \put(33, 30){\vector(-1,0){33}}
    \put(16.7, 32){\makebox(0,0){\tiny\bfseries SDS Image}}

    \put(34.3, 30){\vector(1,0){33}}
    \put(66.6, 30){\vector(-1,0){33}}
    \put(50.0, 32){\makebox(0,0){\tiny\bfseries Latent}}

    \put(66.6, 30){\vector(1,0){33}}
    \put(100, 30){\vector(-1,0){33}}
    \put(83.3, 32){\makebox(0,0){\tiny\bfseries Decoded Latent}}

    \end{overpic}
    \caption{\textbf{Clean latents.} During SDS optimization, images (columns 1--5) contain structured noise, while their latents (columns 6--10) and decoded latents (columns 11--15) remain visually clean.}
    \label{fig:clean_latents}
\end{figure}
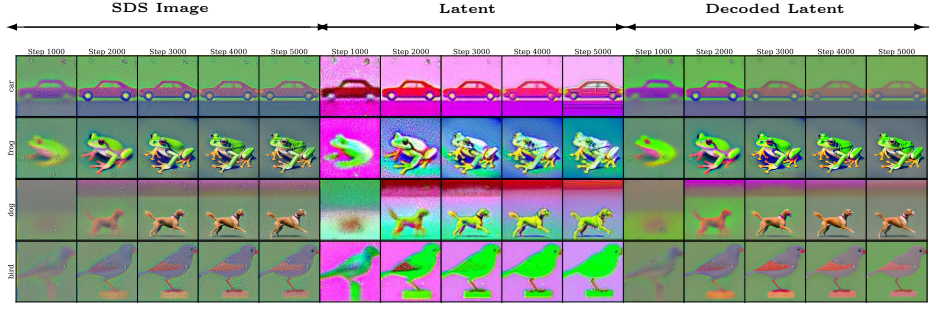

Directly defining or penalizing the noise produced by SDS is difficult, because the artifacts are structured, spatially dependent, and often entangled with semantic image content. Instead, we rely on a more stable signal provided by the VAE itself. As shown in~\cref{fig:clean_latents}, although the optimized image may become noisy, its encoded latent remains semantically meaningful, and decoding this latent produces a much cleaner image with similar overall content.

This suggests that the artifact is not fully visible to the latent diffusion model: the image can drift in pixel-space directions that are weakly constrained by the VAE encoder, while its latent representation remains clean. \ours uses this observation to construct a clean update direction. Rather than adding an explicit reconstruction loss toward decoded latent, which may pull the image backward toward the previous iterate, we decode the next latent SDS step and use it as a clean approximation of where the image should move. This distinction is important: reconstructing the current latent would only pull the image back toward its present VAE projection, whereas PixSDS decodes the next latent SDS step and therefore repairs the update without discarding the semantic direction of SDS.

\subsection{\ours}\label{sec:pixsds_method}

\begin{figure}[t!]
  \centering
  \input{images/pixsds_pseudocode.tex}
\end{figure}
    
\begin{figure}[t!]
  \centering

\begin{overpic}[width=\textwidth]{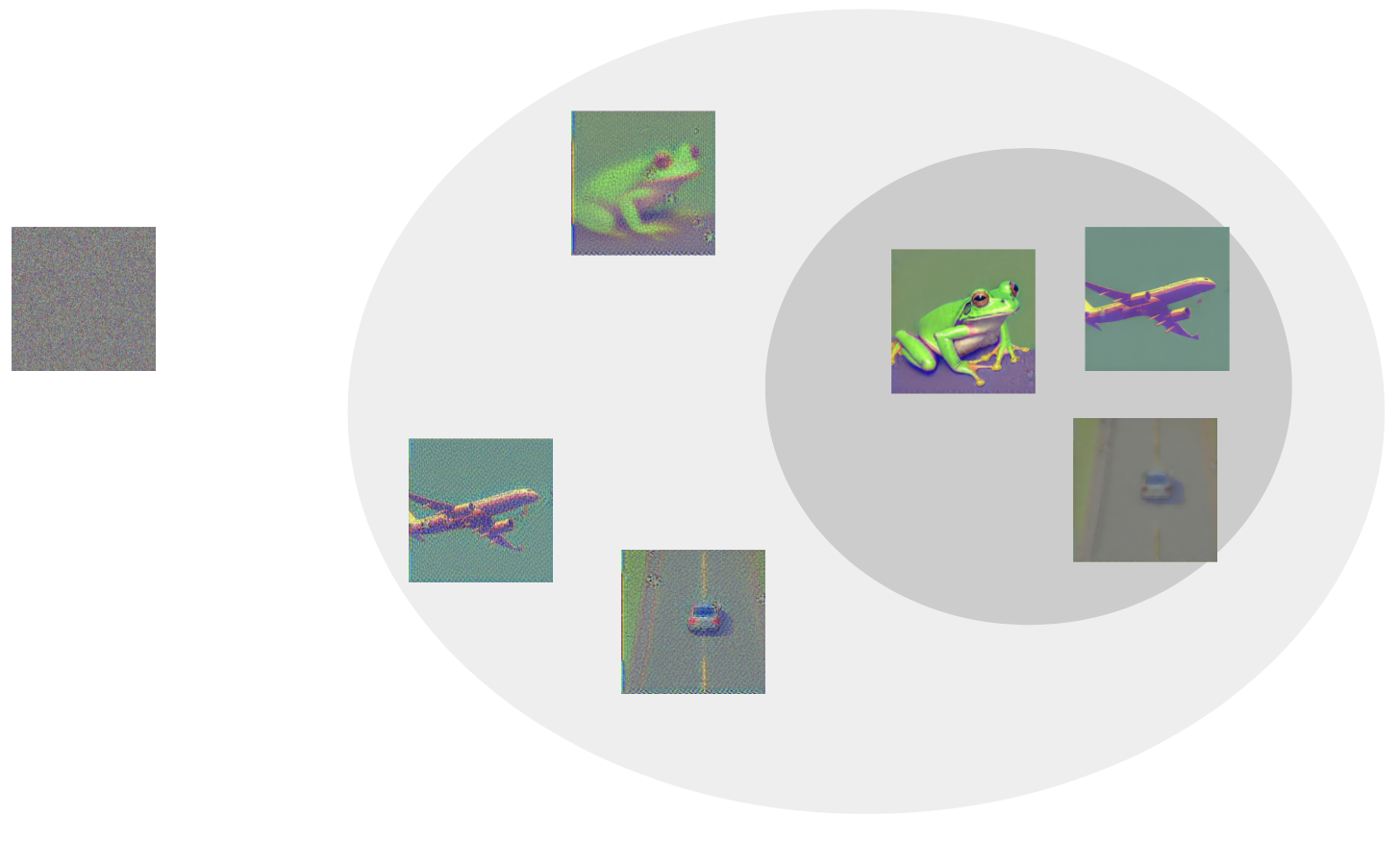}
    \put(6,31){\makebox(0,0){\small\begin{tabular}{c}Optimized\\image\end{tabular}}}
    \put(65,26){\makebox(0,0){\small\begin{tabular}{c}Realistic\\images\end{tabular}}}
    \put(60,11){\makebox(0,0){\small\begin{tabular}{c}Noisy\\images\end{tabular}}}

    \color{green!55!black}
    \linethickness{1.1pt}
    \Line(11,40)(41,50)
    \put(41,50){\vector(3,1){1}}
    \put(16,48){\small\bfseries\color{green!45!black}(a) SDS step}

    \color{cyan!60!blue}
    \linethickness{1.1pt}
    \Line(50,50)(69,44)
    \put(69,44){\vector(2,-1){1}}
    \put(55,50){\small\bfseries\color{cyan!60!blue}(b) Realistic Step}

\end{overpic}
\caption{
\textbf{Overview of \ours.} The \textcolor{green!55!black}{(a)} step points towards noisy images, while \textcolor{cyan!60!blue}{(b)} removes the noise and points towards clean images.
}
\label{fig:pixsds_method_diagram}

\end{figure}

Let $Z \in \mathbb{R}^{C \times H \times W}$ denote the optimized image, and let
\[
    \enc: \mathbb{R}^{C \times H \times W} \rightarrow \mathbb{R}^{c \times h \times w},
    \qquad
    \dec: \mathbb{R}^{c \times h \times w} \rightarrow \mathbb{R}^{C \times H \times W}
\]
denote the VAE encoder and decoder. For clarity, we write $g_{\mathrm{sds}}$ for the image-space SDS update direction applied to $Z$, and $g_{\mathrm{sds}}^{\mathrm{latent}}$ for the corresponding latent-space update direction.

At each iteration, \ours first computes the standard latent SDS update. Instead of applying the image-space update directly, we also take the corresponding step in latent space and decode it:
\begin{equation}\label{eq:pixsds_target}
    \widehat{Z}
    =
    \dec\left(\enc(Z) - \beta g_{\mathrm{sds}}^{\mathrm{latent}}\right),
\end{equation}
where $\beta \geq 0$ controls the size of the latent look-ahead step. The image $\widehat{Z}$ can be interpreted as a clean, VAE-consistent approximation of the image that the latent SDS update is moving toward.

We then define the clean direction
\begin{equation}\label{eq:pixsds_clean_direction}
    g_{\mathrm{clean}} = \widehat{Z} - Z.
\end{equation}
This direction pulls the optimized image toward the decoded latent update, thereby reducing motion in pixel-space directions that are not supported by the VAE latent representation.

To combine this direction with the original SDS update, we dynamically match its per-pixel magnitude to the magnitude of the SDS update. Let
\[
    \cnorm(U)_{i,j}
    =
    \left(\sum_{k=1}^{C} U_{k,i,j}^{2}\right)^{1/2}
\]
denote the per-pixel channel norm, with the resulting $H \times W$ map broadcast over channels when multiplied with a $C \times H \times W$ tensor. The repaired update is
\begin{equation}\label{eq:pixsds_repaired_update}
    \widetilde{g}_{\mathrm{sds}}
    =
    g_{\mathrm{sds}}
    +
    \frac{g_{\mathrm{clean}}}{\cnorm(g_{\mathrm{clean}})}
    \odot
    \cnorm(g_{\mathrm{sds}}),
\end{equation}
The normalization preserves the spatial scale of the original SDS update while replacing part of its noisy pixel-space motion with a cleaner VAE-consistent direction.

Finally, the image is updated using $\widetilde{g}_{\mathrm{sds}}$. Since \ours only modifies the update after the SDS direction has been computed, it does not require changing the diffusion model, the renderer, or the underlying SDS objective. This makes the method lightweight and compatible with other SDS-style objectives. Python-like pseudocode is provided in~\cref{algo:pixsds_pseudocode}, and the update is illustrated in~\cref{fig:pixsds_method_diagram}.

\section{Experiments}\label{sec:experiments}

We evaluate \ours in controlled 2D optimization and in text-to-3D generation pipelines. All experiments are run on a single NVIDIA V100 32GB GPU. The method is implemented in PyTorch using Hugging Face libraries~\cite{wolf2019huggingface}.

\subsection{2D Generation}\label{sec:2d_generation}

\begin{table}
\centering
\scalebox{0.8}{

\begin{tabular}{c|ccccc} \toprule 
Method & FID ($\downarrow$) & CLIP Score ($\uparrow$) & BRISQUE ($\downarrow$) & \makecell{CLIP-IQA \\Quality ($\uparrow$)} & \makecell{CLIP-IQA \\Noisiness ($\uparrow$)} \\ \midrule
SDS & 431.877 & \cellcolor{secondbest}{16.442} & 82.120 & \cellcolor{best}{\textbf{0.859}} & 0.031  \\
SDS-Bridge & 352.007 & 14.895 & 90.030 & 0.429 & 0.174  \\
HIFA & 304.815 & 16.311 & 41.764 & 0.246 & 0.025  \\
NFSD & 296.397 & 15.252 & 67.856 & 0.402 & 0.142  \\
PGC & 293.886 & 15.455 & 75.095 & 0.455 & 0.121  \\
SDI & 424.977 & \cellcolor{best}{\textbf{16.539}} & 72.976 & 0.719 & 0.072  \\
VSD & 321.474 & 15.902 & 85.845 & 0.387 & 0.104  \\
2-step-SDS & 229.813 & 16.225 & \cellcolor{secondbest}{26.448} & 0.760 & \cellcolor{secondbest}{0.479}  \\
PixSDS+SGD & \cellcolor{best}{\textbf{223.021}} & 15.962 & \cellcolor{best}{\textbf{12.850}} & \cellcolor{secondbest}{0.837} & \cellcolor{best}{\textbf{0.590}}  \\
PixSDS+Adam & \cellcolor{secondbest}{229.638} & 15.570 & 31.013 & 0.801 & 0.465  \\
Stable Diffusion & 190.851 & 16.345 & 12.020 & 0.933 & 0.655  \\
\bottomrule
\end{tabular}

}
\caption{
    \textbf{2D Generation Results.}
    The first column lists the method names, while the subsequent columns report the evaluation metrics. \ours achieves the best results in FID, BRISQUE, and CLIP-IQA Noisiness, while remaining competitive in CLIP Score and CLIP-IQA Quality.
}
\label{tab:2d_comparison}
\end{table}

\begin{figure}[h]
    \centering
    \begin{overpic}[width=0.95\textwidth]{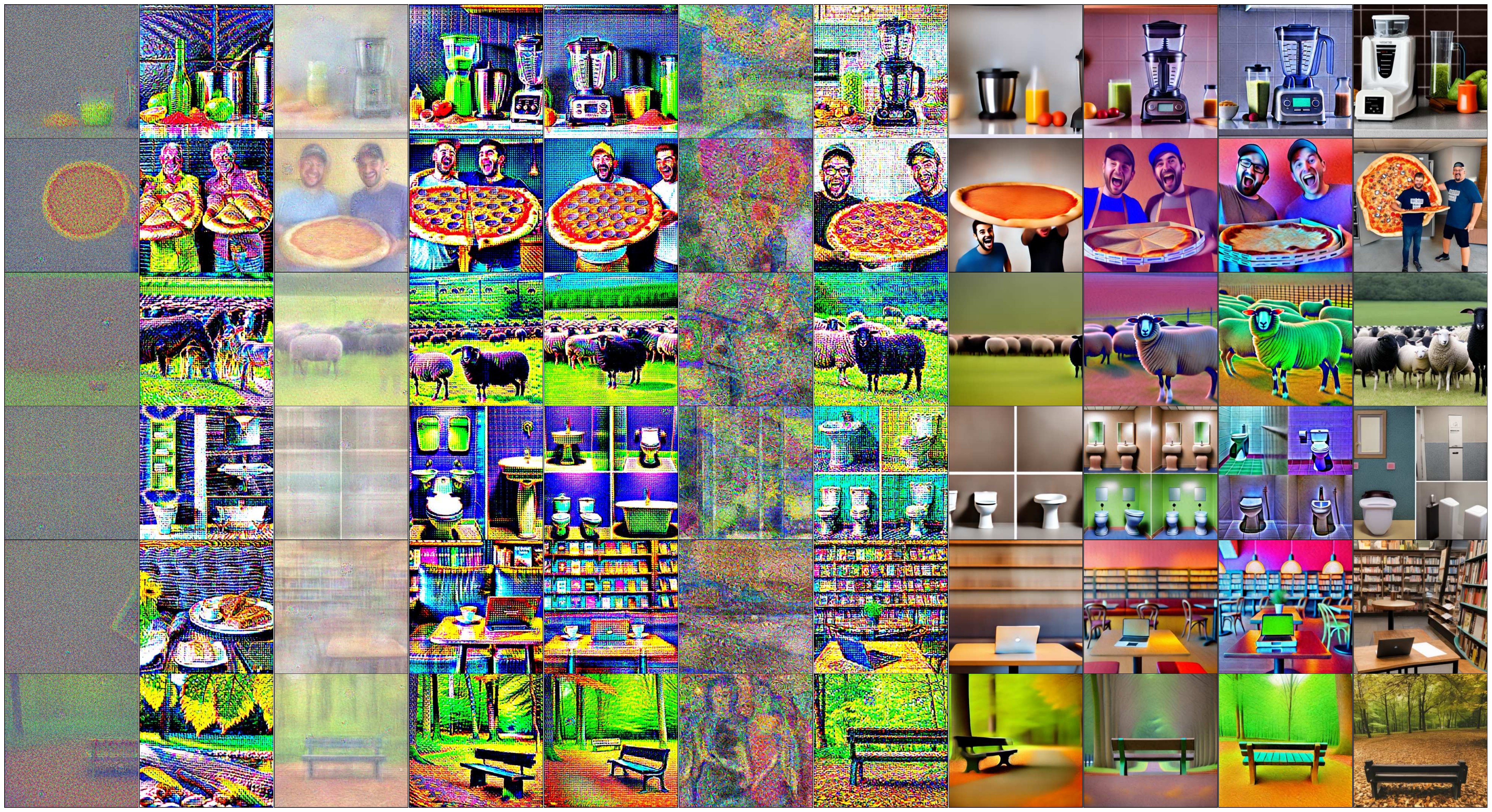}
    \put(3, 55){\tiny{SDS}}
    \put(11, 55){\tiny \shortstack{SDS-\\Bridge}}
    \put(20, 55){\tiny{HiFA}}
    \put(29, 55){\tiny{NFSD}}
    \put(38, 55){\tiny{PGC}}
    \put(48, 55){\tiny{SDI}}
    \put(57, 55){\tiny{VSD}}
    \put(64, 55){\tiny \shortstack{2-step-\\SDS}}
    \put(72, 55){\tiny \shortstack{\ours+\\SGD}}
    \put(81, 55){\tiny \shortstack{\ours+\\Adam}}
    \put(90, 55){\tiny \shortstack{Stable\\Diffusion}}
    \end{overpic}
    \caption{
        \textbf{2D generation comparison.} Each column shows a different method, and each row corresponds to the same text prompt. \ours reduces the structured artifacts that appear in several SDS-style optimization baselines while preserving semantic content.
    }
    \label{fig:2d_comparison}
\end{figure}

We first evaluate \ours in a controlled 2D setting, where the optimized variable is an image rather than a 3D representation. This removes renderer-specific effects and allows us to directly measure whether the proposed gradient repair reduces pixel-space artifacts.

We evaluate two variants of our method, using either \texttt{SGD} or \texttt{Adam} to update the optimized image. In both cases, we use learning rate $0.05$, \texttt{guidance\_scale} = $25.0$ in classifier-free guidance, and set $\beta=0.1$. We initialize each image to the constant value $(0.5,0.5,0.5)$ and optimize it for $N=1{,}000$ steps in \texttt{fp16} precision. Unless otherwise stated, all latent SDS experiments use \texttt{stable-diffusion-}\newline\texttt{2-base}. We use a linearly annealed timestep schedule,
\begin{equation*}
    t
    =
    1000 \cdot
    \operatorname{clip}
    \left(
        1 - \frac{\mathrm{step}}{N},
        0.4,
        1.0
    \right),
\end{equation*}

so that timesteps decrease from $1000$ to $400$ during optimization. We found this schedule to reduce oversaturation; as discussed in the Appendix (Sec. E), the SDS image gradient norms become larger after timestep $400$ for linear timestep annealing.

We randomly sample $100$ captions from MS-COCO 2014~\cite{lin2014microsoft} and generate one image per caption for each method. We report standard image-generation metrics, including FID~\cite{heusel2017gans} and CLIP Score~\cite{hessel2021clipscore}. Since our main goal is artifact reduction, we also report no-reference image quality and noise-related metrics: BRISQUE~\cite{brisque}, CLIP-IQA quality, and CLIP-IQA noisiness. As a direct-sampling reference, we also include images sampled directly from Stable Diffusion. This reference is not an SDS optimization method, but it indicates the quality of the underlying pretrained diffusion model.

Quantitative results are shown in~\cref{tab:2d_comparison}. Among SDS-style optimization methods, \ours achieves the best FID, BRISQUE, and CLIP-IQA Noisiness scores, while remaining competitive in CLIP Score and CLIP-IQA quality. Qualitative results in~\cref{fig:2d_comparison} show the same trend. Vanilla SDS~\cite{poole2022dreamfusion}, SDS-Bridge~\cite{mcallister2024sdsbridge}, HiFA~\cite{zhu2024hifa}, NFSD~\cite{katzir2023nfsd}, PGC~\cite{pan2024pgc}, SDI, and VSD~\cite{wang2023prolificdreamer} often produce visible structured noise or color artifacts. 2-step-SDS~\cite{skorokhodov2025diffusion} produces cleaner images than most baselines, but its results are often oversmoothed. In contrast, \ours reduces artifacts while preserving more local detail, such as the blender buttons in the first row and the chair structures in the fifth row. Although direct Stable Diffusion sampling remains stronger as an image generator, our objective is different: we aim to improve SDS-style optimization rather than replace direct diffusion sampling.

\subsection{3D Generation}\label{sec:3d_generation}

We next test whether the same repair mechanism helps in text-to-3D optimization. We integrate \ours into the second stage of DreamGaussian~\cite{tang2024dreamgaussian}, replacing the SDS update used when SDEdit is disabled. Following the original pipeline, we apply Gaussian smoothing with kernel size $11$ to reduce high-frequency noise. For a fair comparison, we apply the same smoothing to the SDS baseline.

\begin{figure}[h]
    \centering
    \vspace{5pt}
    \includegraphics[width=0.95\textwidth]{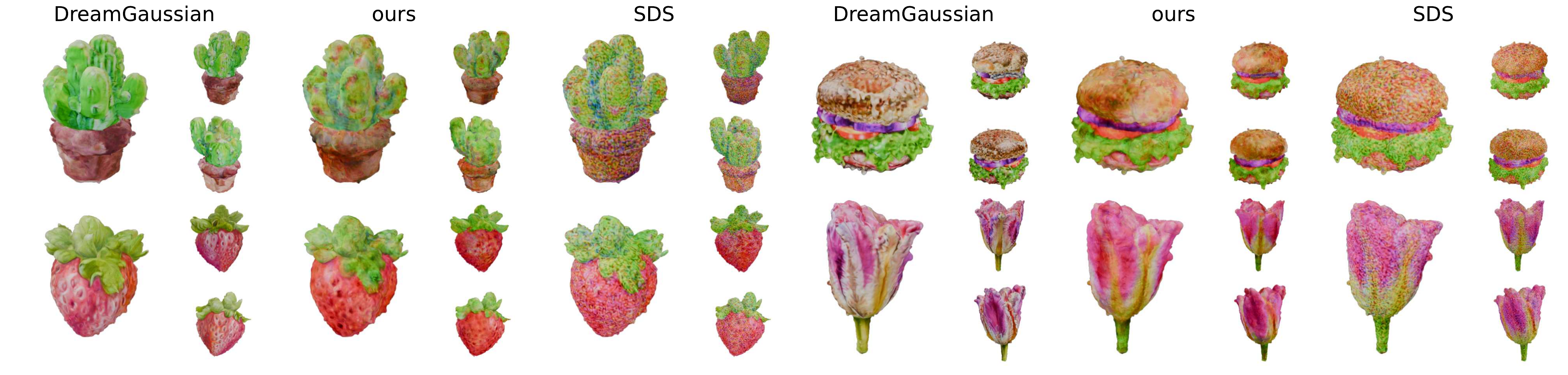}
    \caption{
        \textbf{DreamGaussian comparison.} We apply \ours in the second optimization stage of DreamGaussian~\cite{tang2024dreamgaussian}. Compared with SDS, \ours produces cleaner textures and fewer structured artifacts.
    }
    \label{fig:dreamgaussian}
\end{figure}

As shown in~\cref{fig:dreamgaussian}, \ours produces cleaner results than the SDS baseline, with fewer noisy texture patterns. This experiment also supports the diagnosis from~\cref{sec:diagnosis}: structured artifacts can already be introduced during latent SDS optimization, before mesh extraction or mipmap sampling. Renderer and texture-processing choices may amplify these artifacts, but they are not necessary for the failure mode to appear.

We also apply \ours to LucidDreamer~\cite{liang2024luciddreamer}, as shown in~\cref{fig:luciddreamer}. In this setting, we run optimization for $3000$ steps, sample timesteps from $[0.3,0.8]$, and set $\beta = 100 \cdot \texttt{learning\_rate}$ because latent-space gradients are small during optimization. We do not perform extensive hyperparameter tuning. Even so, \ours consistently reduces the noisy artifacts visible in the baseline. For example, the \textit{football helmet} contains fewer noisy Gaussians inside the object, and the \textit{hamburger} produces fewer floating artifacts around the asset. The \textit{white hair ironman} result is also cleaner around the head, although the white-hair attribute is not fully preserved, suggesting that further tuning may be needed for some prompts.

\begin{figure}[h]
    \centering
    \includegraphics[width=0.95\textwidth]{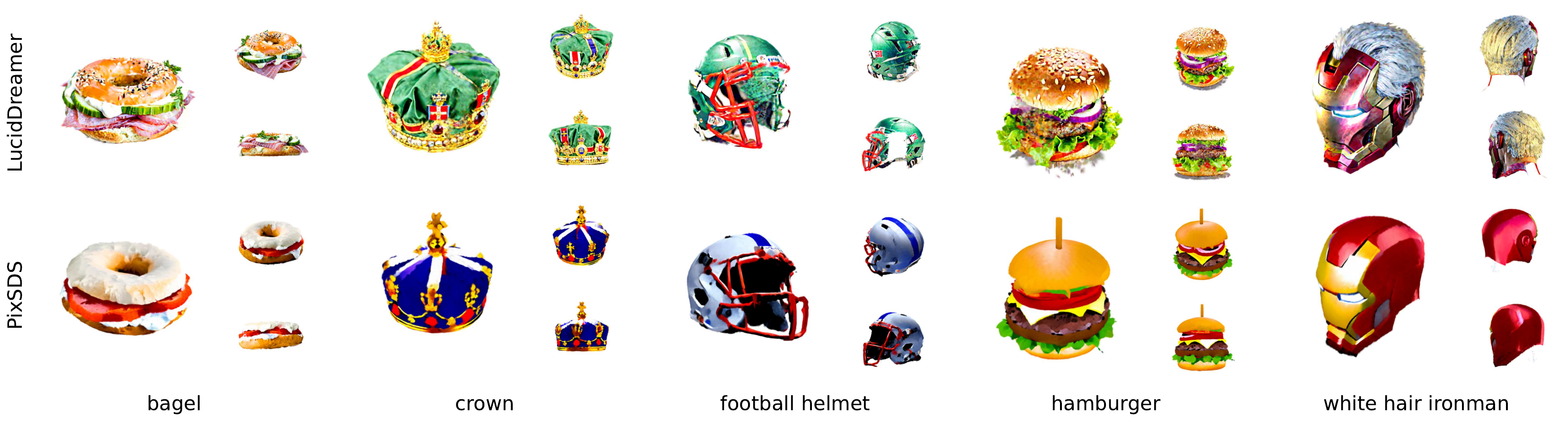}
    \caption{
        \textbf{LucidDreamer comparison.} We integrate \ours into LucidDreamer~\cite{liang2024luciddreamer}. The repaired update reduces floating noisy Gaussians and structured texture artifacts.
    }
    \label{fig:luciddreamer}
\end{figure}

\subsection{Ablation Study}\label{sec:ablation_study}

\begin{figure}[h]
    \centering
    \includegraphics[width=0.95\textwidth]{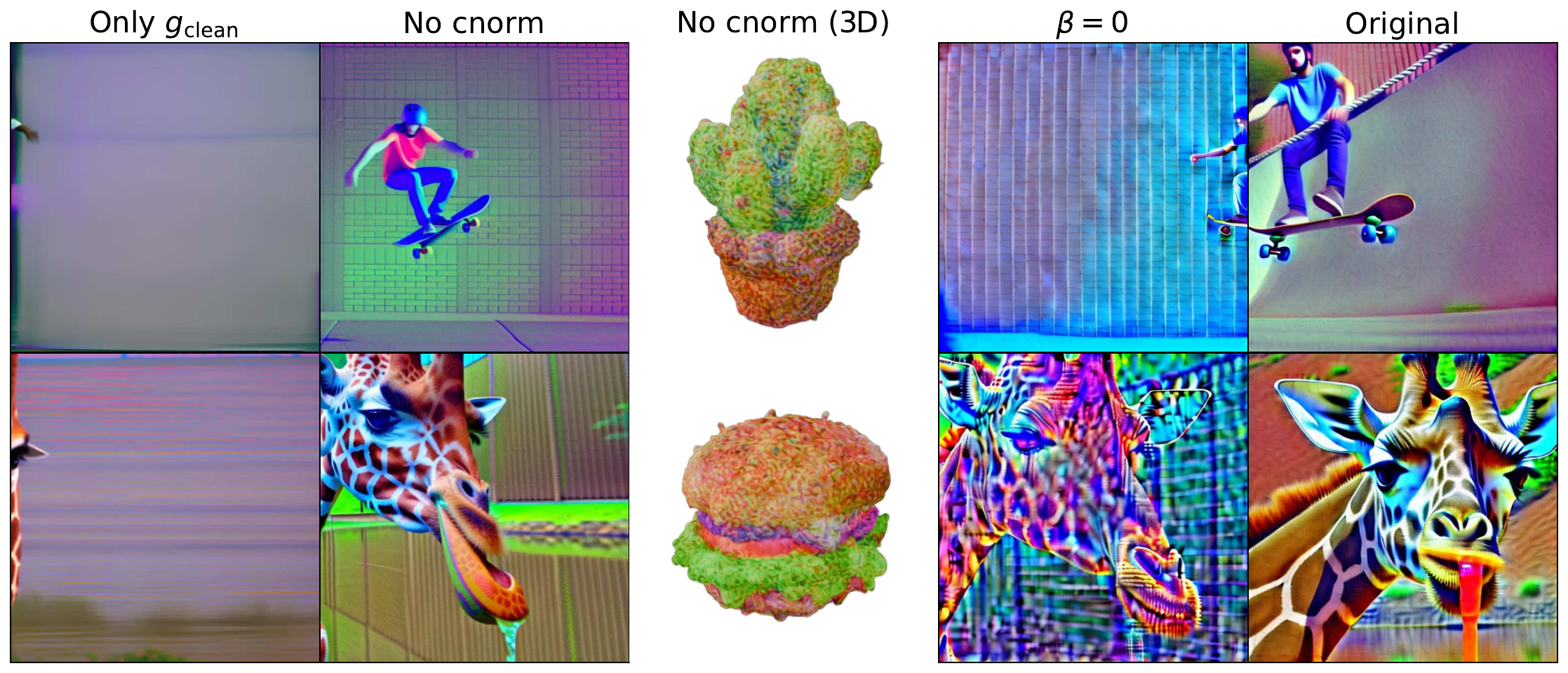}
    \caption{
        \textbf{Ablation study.} We ablate the main design choices of \ours. The full method preserves the SDS update while adding a normalized VAE-consistent clean direction.
    }
    \label{fig:ablation_studies}
\end{figure}

We ablate the key components of \ours in~\cref{fig:ablation_studies}. 
\textbf{(a) Only $\boldsymbol{g_{\mathrm{clean}}}$:} 
we remove the original SDS direction and set $\widetilde{g}_{\mathrm{sds}} = g_{\mathrm{clean}}$. The result is clean but poorly composed, with the object often appearing near the image boundary. This shows that $g_{\mathrm{sds}}$ is still necessary for semantic placement and generation.
\textbf{(b) No channel normalization:} 
we remove the per-pixel normalization and set $\widetilde{g}_{\mathrm{sds}} = g_{\mathrm{sds}} + g_{\mathrm{clean}}$. This can produce clean 2D images, but fails in 3D because the magnitude of $g_{\mathrm{sds}}$ can dominate $g_{\mathrm{clean}}$. The normalization is therefore important for balancing the noisy SDS direction with the clean VAE-consistent direction.
\textbf{(c) $\boldsymbol{\beta=0}$:} 
we remove the latent look-ahead step and set $\widehat{Z}=\dec(\enc(Z))$. This pulls the image toward its current VAE reconstruction, producing clean but unrealistic results. This shows that the clean direction should point toward the next latent SDS update, not merely toward the current decoded latent.
\textbf{(d) Full method:} 
the full \ours update combines $g_{\mathrm{sds}}$ with the normalized clean direction obtained from the decoded latent look-ahead step. This preserves semantic guidance while reducing structured artifacts in both 2D and 3D optimization.

\section{Conclusion}\label{sec:conclusion}

We studied a structured artifact failure mode in latent score distillation. Through controlled experiments, we showed that these artifacts are not fully explained by latent tensor resolution, pixel-space SDS optimization, or renderer-specific effects. Instead, they can arise from VAE-induced pixel drift: the optimized image moves along pixel-space directions that are weakly constrained by the VAE encoder, while its latent representation remains clean and semantically meaningful.
Motivated by this diagnosis, we introduced PixSDS, a lightweight VAE-consistent gradient repair method. PixSDS decodes a latent SDS look-ahead step and uses it to construct a clean pixel-space direction, which is combined with the original SDS update using per-pixel normalization. Experiments in 2D optimization and text-to-3D pipelines show that PixSDS reduces structured artifacts while preserving semantic guidance, without retraining the diffusion model or changing the renderer.
Our results suggest that the geometry of the VAE latent-to-pixel mapping is an important factor in latent SDS optimization, and should be considered when designing future score distillation methods.

\section*{Acknowledgements}
The calculations have been performed using the facilities of the Scientific IT and Application Support Center of EPFL.

\bibliographystyle{splncs04}
\bibliography{main}

\appendix
\section{Noise Proof}

We first prove a more general statement, from which the result for one-dimensional convolutions follows as a special case. Since there is no universally accepted definition for measuring the amount of noise in an image, we introduce a simple proxy for this purpose. This proxy is consistent with our qualitative observations in the 2D generation experiments: its value is $2584.2$ for \ours$+$SGD and $249573.9$ for SDS, meaning that the SDS value is more than $95{\times}$ larger. This agrees with the visual observation that SDS-generated images contain substantially more noise.

\begin{definition}[Noise functional]
Let $D \in \mathbb{R}^{(n-1)\times n}$ be the first-order difference operator defined by
\[
    D_{i,i} = 1, \quad D_{i,i+1} = -1, \quad i = 1, \dots, n-1,
\]
and $D_{i,j} = 0$ otherwise, i.e.,
\[
    D =
    \begin{pmatrix}
    1 & -1 & 0 & \cdots & 0 \\
    0 & 1 & -1 & \cdots & 0 \\
    \vdots & & \ddots & \ddots & \vdots \\
    0 & \cdots & 0 & 1 & -1
    \end{pmatrix}.
\]

Define the functional $N : \mathbb{R}^n \to \mathbb{R}$ by
\[
    N(x) := \|Dx\|^2 = x^T D^T D x
    = \sum_{i=1}^{n-1} (x_i - x_{i+1})^2.
\]

We call $N(x)$ the \textbf{noise functional}.
\end{definition}

\begin{lemma}\label{lemma:mean_inequality}
Let $x \in \mathbb{R}^n$ satisfy
\[
    \sum_{i=1}^n x_i = 0.
\]
Then
\[
    N(x) \ge 2\!\left(1 - \cos\!\left(\frac{\pi}{n}\right)\right)\|x\|^2.
\]
\end{lemma}

\begin{proof}

First observe that
\[
D^T D =
\begin{pmatrix}
1 & -1 & 0 & \cdots & 0 \\
-1 & 2 & -1 & \cdots & 0 \\
0 & -1 & 2 & \ddots & \vdots \\
\vdots & \ddots & \ddots & \ddots & -1 \\
0 & \cdots & 0 & -1 & 1
\end{pmatrix}.
\]
This is the (combinatorial) Laplacian of the path graph on $n$ vertices. Its eigenvalues are given by
\[
\lambda_k = 2 - 2\cos\!\left(\frac{\pi (k-1)}{n}\right),
\quad k = 1, \dots, n,
\]
see, e.g.,~\cite{spielman2009laplacian}.

Let $\{e_1, \dots, e_n\}$ be an orthonormal basis of eigenvectors. The eigenvector corresponding to $\lambda_1 = 0$ is
\[
e_1 = \frac{1}{\sqrt{n}}(1, \dots, 1)^T.
\]
Since $\sum_{i=1}^n x_i = 0$, we have $x \perp e_1$, and hence
\[
x = \sum_{k=2}^n \alpha_k e_k
\quad \text{for some } \alpha_k \in \mathbb{R}.
\]
By orthonormality,
\[
\|x\|^2 = \sum_{k=2}^n \alpha_k^2.
\]

Now,
\[
N(x) = x^T D^T D x
= \sum_{k=2}^n \alpha_k^2 \lambda_k
\ge \lambda_2 \sum_{k=2}^n \alpha_k^2
= \lambda_2 \|x\|^2.
\]
Since
\[
\lambda_2 = 2 - 2\cos\!\left(\frac{\pi}{n}\right),
\]
the result follows.

\end{proof}

\begin{lemma}\label{lemma:noise_increase}
Let $x, y \in \mathbb{R}^n$ with
\[
    x^T D^T D\, y \neq 0.
\]
Then there exists $\alpha \in \mathbb{R}$ such that
\[
    N(x + \alpha y) < N(x).
\]
\end{lemma}

\begin{proof}
Using the definition of $N$, for any $\alpha \in \mathbb{R}$ we have
\begin{align*}
N(x+\alpha y)
&= \|D(x+\alpha y)\|^2 \\
&= \|Dx+\alpha Dy\|^2 \\
&= \|Dx\|^2 + 2\alpha (Dx)^T(Dy) + \alpha^2 \|Dy\|^2 \\
&= N(x) + 2\alpha x^T D^T D y + \alpha^2 N(y).
\end{align*}
Therefore,
\[
    N(x+\alpha y)-N(x)
    = 2\alpha x^T D^T D y + \alpha^2 N(y).
\]

By assumption,
\[
    x^T D^T D y \neq 0.
\]
Hence, for sufficiently small $\alpha$ with sign opposite to
$x^T D^T D y$, the linear term dominates the quadratic term, and so
\[
    N(x+\alpha y)-N(x) < 0.
\]
Thus, there exists $\alpha \in \mathbb{R}$ such that
\[
    N(x+\alpha y) < N(x).
\]
\end{proof}

\begin{lemma}\label{lemma:second_assumption_equivalence}
Let $A \in \mathbb{R}^{n \times n}$ and suppose that $\dim\ker A < n$. The two following statements are equivalent:
\begin{enumerate}
    \item{$\forall x \in \ker A \backslash \{0\}: D^TDx \neq 0$}
    \item{$\forall x \in \ker A \backslash \{0\} \exists z \notin \ker(A)$ such that $x^TD^TDz \neq 0$}
\end{enumerate}
\end{lemma}

\begin{proof}

($\Rightarrow$)

Assume (1) holds. Then we fix an arbitrary $x \in \ker A \backslash \{0\}$. Let $D^TD = L$. Due to symmetry $x^TLz = z^T(Lx)$, and we need to find $z$. Let $H = \{z \in \mathbb{R}^n | z^T (Lx) = 0\}$, it is a vector space and $\dim H < n$ (otherwise $Lx = 0$). The $z$ for the statement (2) is any $z \in \mathbb{R}^n \backslash (\ker A \cup H)$. Note that $\mathbb{R}^n \backslash (\ker A \cup H) \neq \varnothing$ because a vector space is not a union of its two proper subspaces.

($\Leftarrow$)

Assume (1) does not hold, then $\exists \tilde{x} \in \ker A \backslash \{0\}: L\tilde{x} = 0$. But in that case for any $z \in \mathbb{R}^n$ we get $z^T L\tilde{x} = z^T \boldsymbol{0}$ = 0 and (2) does not hold.

\end{proof}

\begin{theorem}[Noise amplification under gradient descent]
\label{thm:noise_linear_operator}

Let $A \in \mathbb{R}^{m \times n}$ and suppose that:
\begin{enumerate}
    \item $r = \dim \ker(A), r \ge 2$;
    \item $\forall x \in \ker(A) \backslash \{0\}: D^TDx \neq \boldsymbol{0}$.
\end{enumerate}

Let $x^*$ be a minimizer of
\[
    \min_{x \in \mathbb{R}^n} \; \frac{1}{2}\|Ax\|^2,
\]
obtained via gradient descent.

Then for every $C > 0$, there exists $x \in \mathbb{R}^n$ such that
\[
    N(x^*) \ge C
    \quad \text{and} \quad
    N(x) < N(x^*).
\]

In particular, gradient descent can increase the noise level.
\end{theorem}

\begin{proof}
Let
\[
    L(x) = \frac12 \|Ax\|^2 = \frac12 x^T A^T A x.
\]
Then
\[
    \nabla L(x) = A^T A x.
\]

Let $\{e_1,\dots,e_n\}$ be an orthonormal eigenbasis of $A^T A$, with
\[
    A^T A e_i = \lambda_i e_i,
\]
where
\[
    \lambda_1=\cdots=\lambda_r=0,
    \qquad
    \lambda_i>0 \quad \text{for } i>r.
\]
By assumption (1), $r \ge 2$.

For an initial point
\[
    x_0 = \sum_{i=1}^n c_i e_i,
\]
One gradient descent with step size $\gamma>0$ gives
\begin{align*}
    x_{1} &= x_0 - \gamma A^T A x_0
    = \sum_{i=1}^n c_i e_i - \gamma \sum_{i=1}^n c_i A^TA e_i = \\
    & = \sum_{i=1}^n c_i (1 - \gamma \lambda_i) e_i
\end{align*}

Hence
\[
    x_t
    =
    \sum_{i=1}^n c_i(1-\gamma\lambda_i)^t e_i.
\]
Choose $\gamma$ such that
\[
    0 < \gamma < \frac{2}{\max_{i>m}\lambda_i}.
\]
Then
\[
    \lim_{t\to\infty} (1-\gamma\lambda_i)^t = 0
    \quad \text{for } i>m,
\]
and therefore
\[
    x^*
    =
    \lim_{t\to\infty} x_t
    =
    \sum_{i=1}^m c_i e_i
    \in \ker(A).
\]

Since $\dim \ker(A)\ge 2$, there exists a nonzero vector
\[
    u \in \ker(A)
\]
such that
\[
    \sum_{i=1}^n u_i = 0.
\]
By \cref{lemma:mean_inequality},
\[
    N(u)
    \ge
    2\left(1-\cos\left(\frac{\pi}{n}\right)\right)\|u\|^2
    >0.
\]

Using assumption (2) and \cref{lemma:second_assumption_equivalence}, there exists $z \notin \ker(A)$ such that
\[
    u^T D^T D z \neq 0.
\]
Therefore, by \cref{lemma:noise_increase}, there exists $\alpha\in\mathbb{R}$ such that
\[
    N(u+\alpha z) < N(u).
\]
Now choose $s>0$ such that
\[
    s^2 N(u) \ge C
    \qquad \text{and} \qquad
    s^2 N(u+\alpha z) < C.
\]
Such an $s$ exists because $N(u+\alpha z)<N(u)$.

Set
\[
    x_0 := s(u+\alpha z).
\]
Since $z\notin\ker(A)$, the non-kernel component of $x_0$ vanishes under gradient descent, while the kernel component remains. Hence the limit point is
\[
    x^* = su.
\]
Consequently,
\[
    N(x_0)
    =
    s^2 N(u+\alpha z)
    < C,
\]
whereas
\[
    N(x^*)
    =
    s^2 N(u)
    \ge C.
\]
Thus, there exists an initial vector whose noise is below $C$, but whose gradient descent limit has noise at least $C$.
\end{proof}

\begin{corollary}\label{corollary:noise_in_convolution}
Let $w$ with weights $\{w_1, \dots, w_m\}$ is a $1d{-}$convolution over $x \in \mathbb{R}^n$ satisfying

\begin{enumerate}
    \item $m \ge 3$;
    \item $\sum\limits_{i=1}^m w_i \neq 0$.
\end{enumerate}

Then the \cref{thm:noise_linear_operator} holds for that operation.

\end{corollary}

\begin{proof}

We can present $w$ as a matrix $A \in \mathbb{R}^{(n -m + 1) \times n}$

\begin{equation*}
A = \begin{pmatrix}
w_1 & w_2 & 0 & \dots & w_m & 0 & \dots & 0 \\
0 & w_1 & w_2 & \dots & w_{m - 1} & w_m & \dots & 0 \\
& & & & \dots \\
0 & 0 & 0 & \dots & w_1 & w_2 & \dots & w_m \\
\end{pmatrix}
\end{equation*}

Note that $\operatorname{rank} A \le n - m + 1$ due to the matrix dimensions. Therefore, $\dim \ker A \ge n - (n - m + 1) = m - 1$. If $m \ge 3$ then $\dim \ker A \ge 2$.

Find vectors $x$ that satisfy $D^TDx = \boldsymbol{0}$. 
\[
D^TDx = \boldsymbol{0} \Leftrightarrow x^TD^TDx = 0 \Leftrightarrow ||Dx||^2 = 0 \Leftrightarrow Dx = \boldsymbol{0}.
\]

This is exactly when $x_1 = x_2 = \dots = x_n$, i.e. $x = c\boldsymbol{1}$ for $c \in \mathbb{R}$. Note that due to the condition (2) for any $c \neq 0: c\boldsymbol{1} \notin \ker A$. It means that $\forall x \in \mathbb{R}^n \backslash \{0\} | Lx = \boldsymbol{0} \Rightarrow x \notin \ker A$. Therefore, the second condition of the theorem holds.
    
\end{proof}

\section{VAE Optimization Loss Landscape}

\begin{figure}[t!]
\begin{minipage}{0.49\textwidth}
\centering
\includegraphics[width=\textwidth]{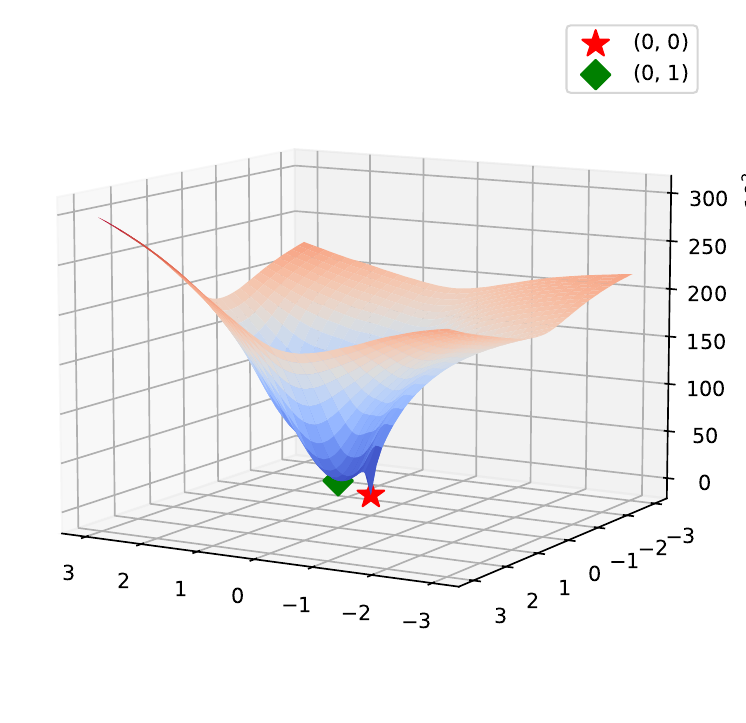}
\caption{(a) Specified directions}
\label{fig:vae_loss_landscape_1}
\end{minipage}
\hfill
\begin{minipage}{0.49\textwidth}
\centering
\includegraphics[width=\textwidth]{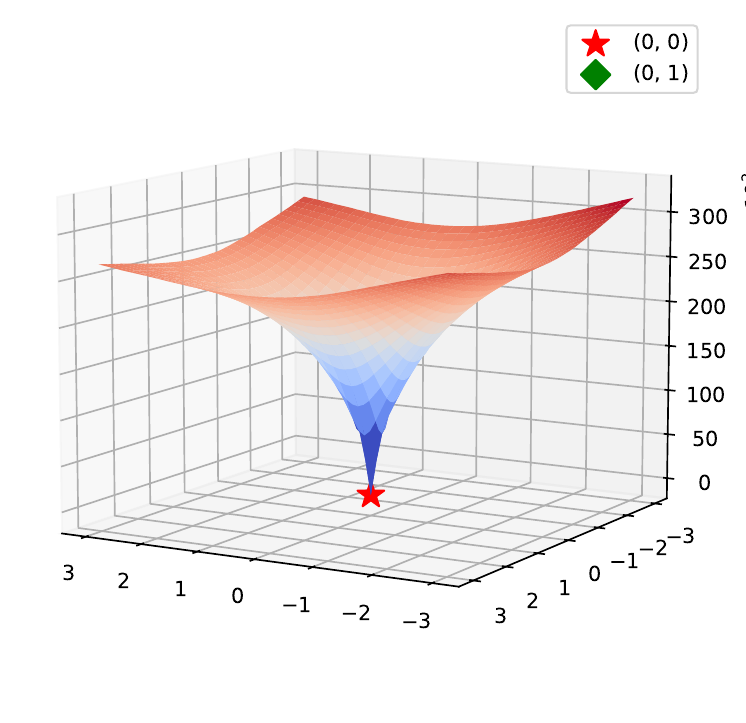}
\caption{(b) Random directions}
\label{fig:vae_loss_landscape_2}
\end{minipage}
\caption{
\textbf{Visualization of the VAE-only optimization loss landscape.} The red star denotes the original image $X$, while the green diamond denotes an optimized solution $Z$ satisfying $f(Z) \approx f(X)$.
}
\label{fig:vae_loss_landscape}
\end{figure}

We further analyze the loss landscape of the VAE-only optimization experiment described in
Sec. 3 (Diagnosing SDS Artifacts). We follow the notation from that experiment. For visualization, we center the coordinate system at the original image $X$ and evaluate the loss on a two-dimensional affine subspace. Specifically, given two directions $v_1, v_2 \in \mathbb{R}^{C \times H \times W}$, we compute the loss at points of the form

\[
X + \alpha_1 v_1 + \alpha_2 v_2,
\qquad
(\alpha_1, \alpha_2) \in [-1, 1]^2 .
\]

This allows us to inspect how the optimization objective behaves around both the clean image $X$ and alternative solutions $Z$ that reconstruct to a similar latent representation.

Fig.~\ref{fig:vae_loss_landscape} shows the resulting landscapes for two choices of directions. In Fig.~\ref{fig:vae_loss_landscape_1}, we use specified directions constructed from optimized solutions: $v_i = Z_i - X$, where each $Z_i$ is obtained by VAE-only optimization and satisfies $f(Z_i) \approx f(X)$. In this case, the landscape contains local minima away from the origin. These minima correspond to noisy images that remain close to $X$ in the VAE latent space. Importantly, moving from such a solution $Z_i$ back toward the clean image $X$ requires first increasing the objective, which can prevent gradient descent from recovering the global solution once it reaches the local minimum.

In contrast, in Fig.~\ref{fig:vae_loss_landscape_2}, we sample random directions $v_i \sim \mathcal{U}[-1,1]$. Along these directions, we do not observe the same spurious local minima. This suggests that the problematic solutions are not arbitrary perturbations of the image, but instead have a specific structure aligned with directions that are weakly constrained by the VAE representation.

Overall, this visualization supports our hypothesis that VAE-only optimization can admit structured noisy solutions $Z$ for which $f(Z) \approx f(X)$. Such solutions form undesirable basins in the pixel-space loss landscape. Although the clean image $X$ remains the global optimum, these basins can trap gradient-based optimization and make recovery of the clean solution difficult.

\section{$\beta$ Hyperparameter Study}

We study the sensitivity of our method to the choice of the hyperparameter $\beta$. To this end, we vary $\beta$ in the range $0 \leq \beta \leq 10$ and generate images for each value, while keeping all other settings identical to the 2D generation experiments described in Sec. 5.1.

As shown in~\cref{fig:beta_sensitivity}, the method produces realistic and visually clean images across a broad range of values, particularly for $\beta \in [0.075, 1.0]$. This suggests that PixSDS is not overly sensitive to the exact choice of $\beta$ within this interval. At very small values, the effect of the clean-direction correction becomes limited, while excessively large values may overemphasize the look-ahead direction and lead to less stable updates. In our experiments, we therefore use $\beta$ within the stable range identified above.

\begin{figure}[h]
\centering
\includegraphics[width=0.95\textwidth]{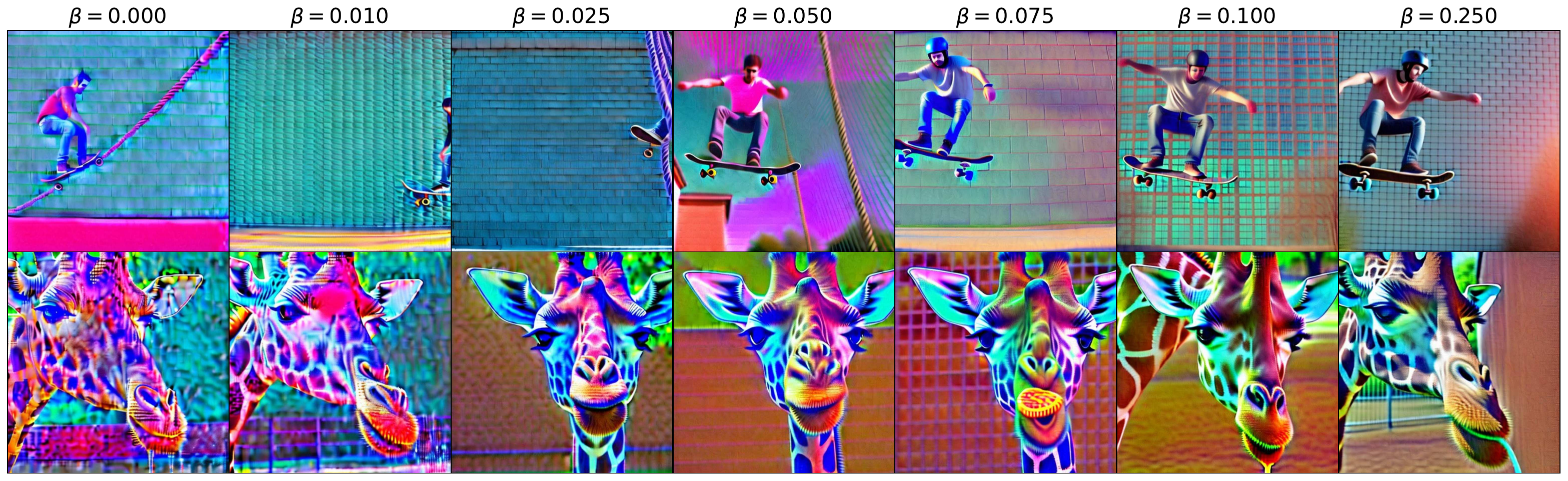}
\includegraphics[width=0.95\textwidth]{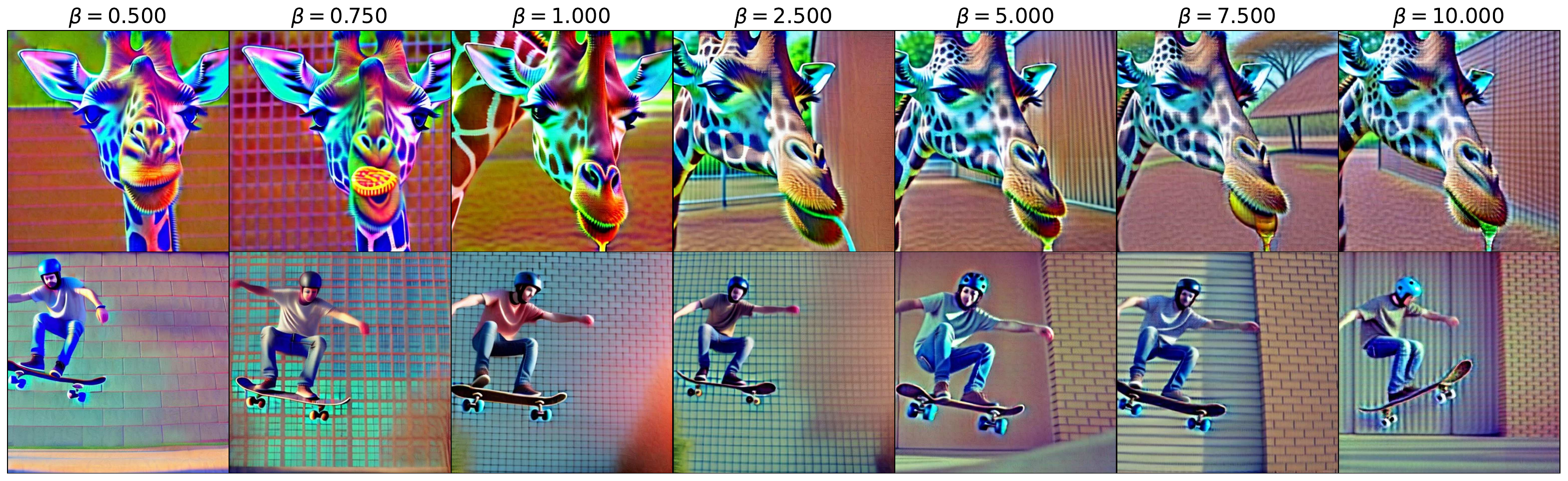}
\caption{\textbf{Sensitivity to $\beta$.} We vary the look-ahead hyperparameter $\beta$ while keeping all other settings fixed. The results remain realistic and visually clean for a wide range of values, especially for $\beta \in [0.075, 1.0]$, indicating that the method is robust to moderate changes in this parameter.}
\label{fig:beta_sensitivity}
\end{figure}

\section{\ours for Stable Diffusion 3}

We further evaluate whether \ours can be applied beyond the standard latent diffusion setting by testing it with Stable Diffusion 3~\cite{esser2024scalingrectifiedflowtransformers}, which is based on Rectified Flow~\cite{liu2022flowstraightfastlearning}. We use the same hyperparameters as in the 2D generation experiments described in Sec. 5.1, except that we set the number of optimization steps to $N=500$.

As shown in~\cref{fig:stable_diffusion_3}, \ours produces clean and realistic images with Stable Diffusion 3. These results suggest that the proposed correction is not specific to a single diffusion backbone, and can be integrated with different generative formulations, including rectified-flow-based models.

\begin{figure}[h]
\centering
\includegraphics[width=0.95\textwidth]{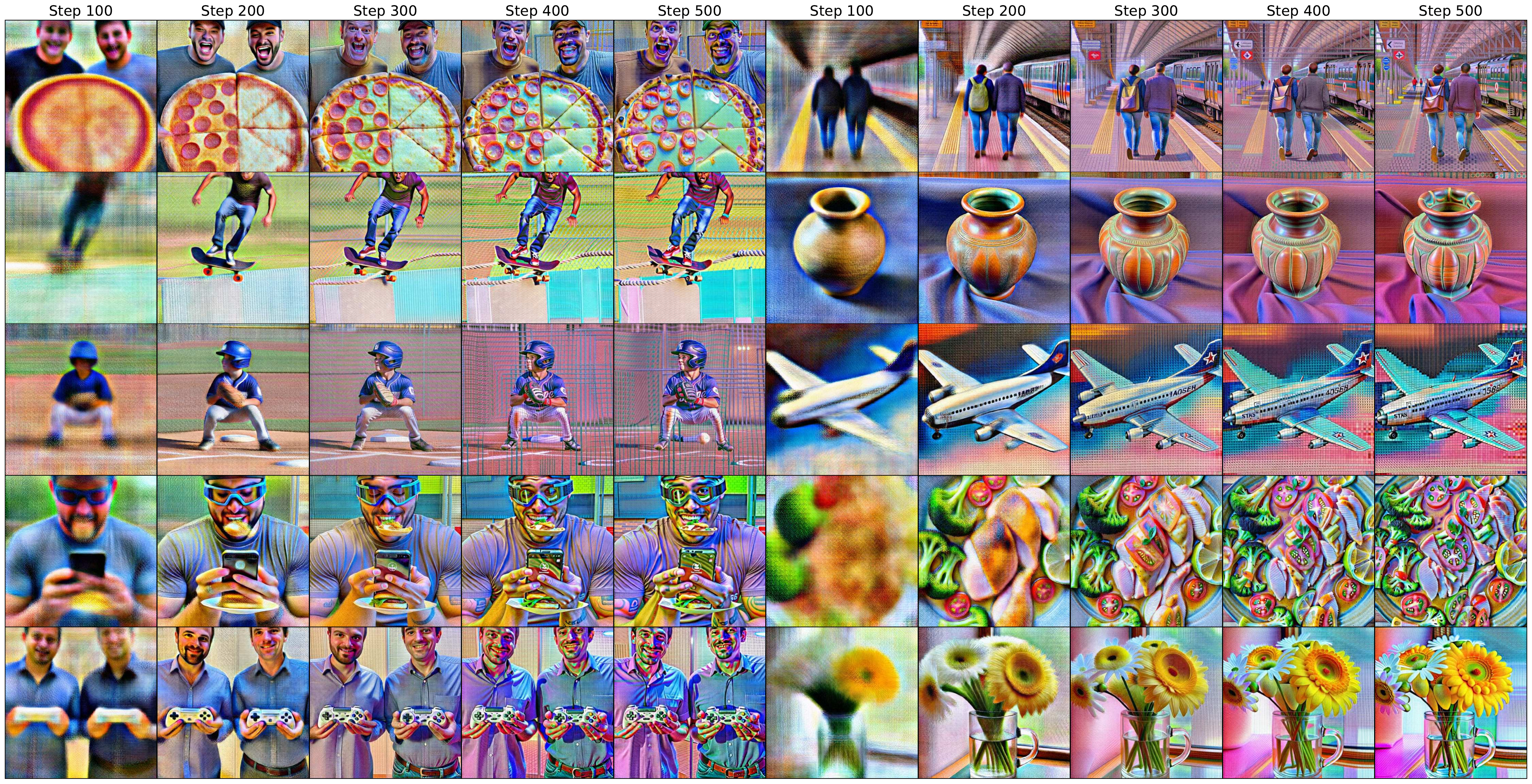}
\caption{\textbf{\ours with Stable Diffusion 3.} Qualitative examples generated using \ours with Stable Diffusion 3. The results remain clean and realistic, suggesting that the method can be applied beyond standard latent diffusion models.}
\label{fig:stable_diffusion_3}
\end{figure}

\section{Timestep Strategy}

We compare our proposed timestep strategy with a standard linear annealing schedule using the 2D generation setup described in Sec. 5.1. We randomly sample $M=10$ prompts from MS-COCO~\cite{lin2014microsoft} and evaluate both timestep strategies for each prompt, keeping all other hyperparameters fixed. During optimization, we record the latent-space and image-space gradient norms at each step and aggregate the results across prompts.

\Cref{fig:timesteps_comparison} shows qualitative generations obtained with the two schedules. Compared to linear timestep annealing, our proposed strategy produces images with more balanced colors and fewer oversaturated regions. In contrast, linear annealing often leads to stronger color biases, such as overemphasized green backgrounds (second and ninth rows).

\Cref{fig:grad_norms} reports the image-space gradient norms throughout optimization. Linear timestep annealing produces substantially larger gradient norms after approximately step $600$, which corresponds to diffusion timestep $t=400$ under a $1000$-step schedule. Such large gradients can make the optimization less stable, as a single step may introduce excessive changes in image space. By comparison, our timestep strategy maintains more moderate gradient magnitudes during the later stages of optimization, helping preserve visual stability and improve generation quality.

\begin{figure}[t!]
\centering
\includegraphics[width=0.95\textwidth]{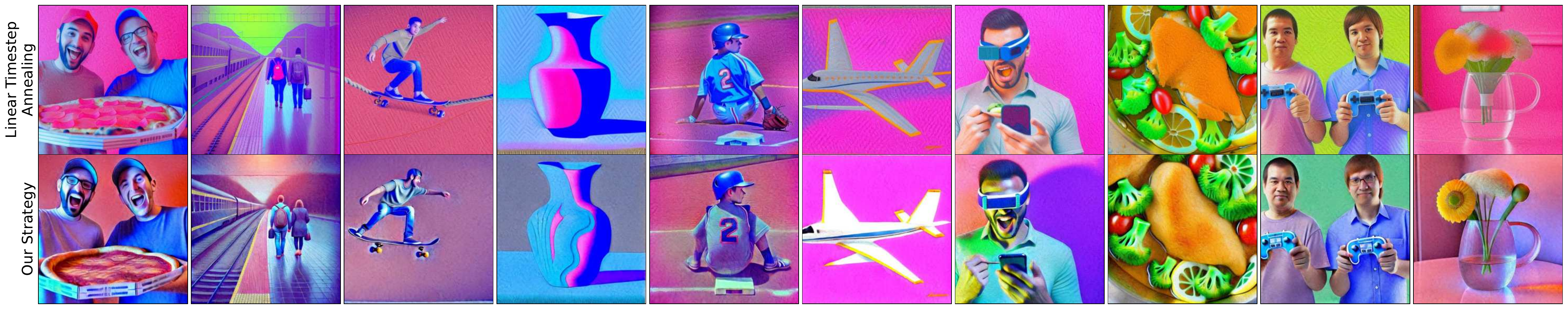}
\caption{
\textbf{Comparison of timestep annealing strategies.}
The top row shows results obtained with linear timestep annealing, while the bottom row shows results obtained with our proposed timestep strategy. Our schedule produces more balanced colors and fewer oversaturated regions.
}
\label{fig:timesteps_comparison}
\end{figure}

\begin{figure}[t!]
\centering
\includegraphics[width=0.60\textwidth]{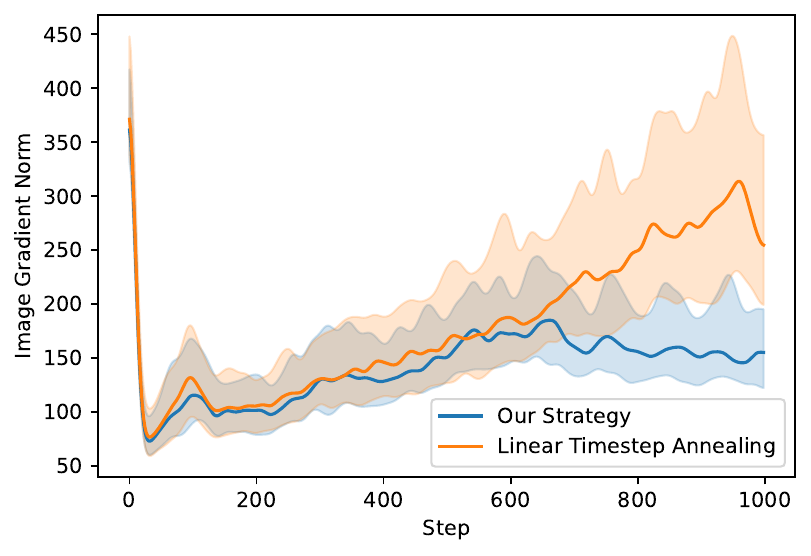}
\caption{
\textbf{Image-space gradient norms during generation.}
We compare linear timestep annealing with our proposed timestep strategy, averaged across $10$ MS-COCO prompts. Solid lines indicate the median gradient norm at each optimization step, and shaded regions indicate the interquartile range between the $0.25$ and $0.75$ quantiles.
}
\label{fig:grad_norms}
\end{figure}

\end{document}